\documentclass[letterpaper, 10 pt, conference]{ieeeconf}  

\IEEEoverridecommandlockouts                              

\usepackage{graphicx} 
\usepackage{amsmath} 
\usepackage{amssymb}  
\newtheorem{theorem}{Theorem}

\title{\LARGE \bf
Design of A Two-point Steering Path Planner Using Geometric Control
}

\author{Yunlong Huang$^{1}$
\thanks{*This work was not supported by any organization}
\thanks{$^{1}$Yunlong Huang is with Controls $\&$ Automated Systems, Research $\&$ Advanced Engineering,
        Ford Motor Company, 2000 Rotunda Drive, MI, USA
        {\tt\small yunlonog@gmail.com}}%
}

\begin{document}

\maketitle
\thispagestyle{empty}
\pagestyle{empty}

\begin{abstract}

For lateral vehicle dynamics, planning trajectories for lane-keeping and lane-change can be generalized as a path planning task to stabilize a vehicle onto a target lane, which is a fundamental element in nowadays autonomous driving systems. On the other hand, two-point steering for lane-change and lane-keeping has been investigated by researchers from psychology as a sensorimotor mechanism of human drivers. In the first part of this paper, using knowledge of geometric control, we will first design a path planner which satisfies five design objectives: generalization for different vehicle models, convergence to the target lane, optimality, safety in lane-change maneuver and low computational complexity. Later, based on this path planner, a two-point steering path planner will be proposed and it will be proved rigorously that this two-point steering path planner possesses the advantage--steering radius of the planned trajectory is smaller than the intrinsic radius of reference line of the target lane. This advantage is also described as ``corner-cutting'' in driving. The smaller driving radius of the trajectory will result in higher vehicle speed along the winding roads and more comfortness for the passengers.
\end{abstract}

\section{INTRODUCTION}

In today's autonomous driving system (e.g., Baidu Apollo \cite{Apollo}), given the reference line of a target lane from a HD (High-Definition) map, to have the vehicle converge to target lane with a lateral vehicle dynamics is a basic path planning task. For a model-based path planner, the first challenge is its generalization for different vehicle models. The lateral vehicle dynamics is a nonlinear system and the complexity of the modeling will arise from single track models (dynamic bicycle model and kinematic bicycle model \cite{Bicycle_models}) to more sophisticated two track vehicle models \cite{Vehicle_planar_dynamics}. Compared to some existing path planners which are designed for some specific vehicle models \cite{Snider}, our path planner as a methodology which can be applied to various vehicle models will be more adorable from the view of design and development of the autonomous driving systems. 

Besides the generalization of the path planner, convergence properties of path planner, such as convergence rate to the target lane and avoidance of hazardous damping behaviors of the planned paths, is also a key objective in the design. As the convergence rate will determine the duration of a lane-change maneuver and avoidance of hazardous damping behaviors is under the consideration of safety of lane-change maneuver. To steering the vehicle in a comfortable manner for the passengers, the path planner needs to be a solution of a related optimal control problem. In today's application of lane-change path planner, the vehicle can abort its on-going lane-change maneuver when its surrounding traffic environment turns to be unsafe for this maneuver. This abortion of on-going lane-change maneuver requests the states of the lane-change trajectory (e.g., orientation of the vehicle velocity and its time-derivative) should be bounded which is a safety objective in design. In the planning module of nowadays autonomous driving system (e.g., Baidu Apollo, up to ver. $3.0.0$ \cite{Apollo}) which runs at $10$Hz and the planning horizon is 8 seconds, it is challenging for the hardware platform of the system to accurately solve a rigorously formulated path planning related numerical optimization problem (e.g., Model Predictive Control problem) online. Therefore, the fifth design objective of the path planner is low computational complexity and the complexity of our path planner turns to be linear with planning horizon.  

Besides the research from engineering which focuses on designing path planners for autonomous systems, researchers from psychology \cite{Land} \cite{Salvucci} are also investigating the sensorimotor mechanism of human drivers in lane-keeping and lane-change maneuvers and this mechanism is named as two-point steering. In different research experiments, two-point steering exhibits some advantages. For example, the steering radius of the planned trajectory is smaller than the intrinsic steering radius of the target lane. Designing a path planner for two-point steering and integrating it into the autonomous driving system is beneficial for the system in two-fold: first the performance of the path planner is directly improved and secondly it makes the system more similar to a human driver which makes it more acceptable in society.  

In Section II, the projection of a given vehicle dynamics into an orthonormal frame will be introduced as the geometric control background. Using this projection, our path planner, as a methodology can be applied to various vehicle dynamics. The optimality and convergence of the path planner will be investigated in Section III. The safety design of the path planner (abortion of lane-change maneuver) is in Section IV. Based on this path planner, we will design a two-point steering path planner and its advantage will be proved in Section V. Future work will be proposed in Section IV.

\section{Projection and Orthonormal Frame}

Orthonormal frame was introduced for path planner design in previous work already (e.g., \cite{Zhang}). The purpose of this section is to emphasize that by projecting the vehicle state space into an orthonormal frame, our path planner is applicable to different vehicle models, if the orientation of the vehicle velocity is locally controllable \cite{Khalil}.
Given a kinematic bicycle model, as illustrated in Figure \ref{fig:kbm}, \\
\begin{figure}[t]
      \centering
      \includegraphics[scale=0.3]{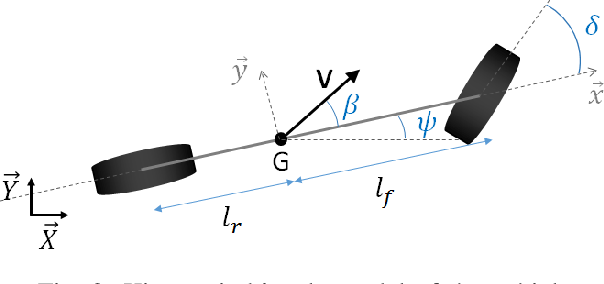}
      \caption{Kinematic bicycle model}
      \label{fig:kbm}
\end{figure}
\begin{align}\label{eq:kbm}
    \dot{x} &= v \cos(\psi + \beta) \nonumber  \\
    \dot{y} &= v \sin(\psi + \beta) \nonumber \\
    \dot{\psi} &= \frac{v}{l_{r}} \sin(\beta) \nonumber \\
    \beta &= \tan^{-1} \left(\frac{l_{r} \tan(\delta)}{l_{f} + l_{r}}\right) \nonumber \\
    \dot{\delta} &= u
\end{align}
where $(x, y)$ is the coordinate of the vehicle's center of gravity $G$ in the coordinate reference system of $(\vec{X}, \vec{Y})$ , $\psi$ is the orientation of the vehicle frame, $\beta$ is the slip angle of the center of gravity, $\delta$ is the angle of the front-wheel relative to the body frame, $u$ is the turning rate of the front-wheel which is also the control input and $l_f$ ($l_r$) is the length between the center of gravity and the front-axle (rear-axle) of the vehicle. 

In a HD map, the center-line of a lane will be used as the \textit{reference line} for path planning. Given the reference line and the vehicle state, assuming that the point on the reference line which is closest to the vehicle is unique (i.e. \textit{shadow point} $s$), let $\vec{r}_{s}$ denote the position vector of the shadow point in Figure \ref{fig:config}, let $\vec{x}_{s}$ denote the unit tangent vector along the reference line at the shadow point and let $\vec{y}_{s}$ denote the unit normal vector. We use the convection that a unit normal vector completes a right-handed orthonormal frame with the corresponding unit tangent vector. The reference line can be described as
\begin{align}
    \dot{\vec{r}}_{s} &= v_{s} \vec{x}_{s} \\ \nonumber
    \dot{\vec{x}}_{s} &= \vec{y}_{s} v_{s} \kappa \\ \nonumber
    \dot{\vec{y}}_{s} &= -\vec{x}_{s} v_{s} \kappa,
\end{align}
where $v_{s}$ is the speed which can be planned by a longitudinal planner along the reference line and $\kappa$ is the curvature.

Let $\vec{r}_{v}$ denote the position vector of the vehicle, let $\vec{x}_{v}$ denote the unit tangent vector, let $\vec{y}_{v}$ denote the unit normal vector, $v$ is the speed of the vehicle, $\omega$ is the angular velocity of the orientation of the vehicle velocity. We have the orthonormal frame about the vehicle:
\begin{align}
    \dot{\vec{r}}_{v} &= v \vec{x}_{v} \\ \nonumber
    \dot{\vec{x}}_{v} &= \vec{y}_{v} \omega \\ \nonumber
    \dot{\vec{y}}_{v} &= -\vec{x}_{v} \omega.
\end{align}
\begin{figure}[t]
      \centering
      \includegraphics[scale=0.3]{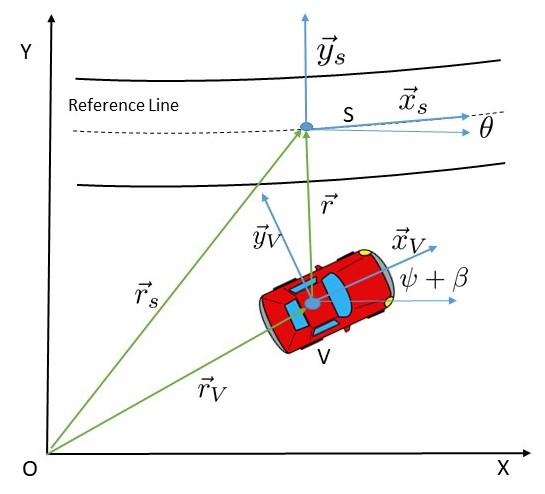}
      \caption{Orthonormal frames and reference line}
      \label{fig:config}
\end{figure}
This orthonormal frame of the vehicle is the projection of the state space of vehicle $(x, y, \phi, \beta, \delta)$ in \eqref{eq:kbm} to a state space of $(x, y, \psi + \beta)$, where
\begin{equation}
\vec{r}_{v} = \begin{bmatrix} x\\ y\end{bmatrix}, \vec{x}_{v} = \begin{bmatrix} \cos(\psi + \beta) \\ \sin(\psi + \beta)\end{bmatrix}, \vec{y}_{v} = \begin{bmatrix} -\sin(\psi + \beta) \\ \cos(\psi + \beta)\end{bmatrix}
\end{equation}
and 
\begin{equation}
    \omega \equiv \dot{\psi} + \dot{\beta} = \frac{v}{l_r} \sin(\beta) + g(\delta) u 
\end{equation}
where 
\begin{equation}
    g(\delta) = \frac{\frac{l_r}{l_f + l_r}}{\left(1 + \left(\frac{l_r \tan(\delta)}{l_{f} + l_{r}}\right)^{2} \right) \cos(\delta)^{2}}.
\end{equation}
As $\psi + \beta$ and $u$ are scalars, the velocity orientation ($\psi + \beta$) is locally controllable \cite{Khalil}.
We define 
\begin{equation}
\vec{r} \equiv \vec{r}_{s} - \vec{r}_{v}    
\end{equation}
to be the vector from the vehicle pointing to the shadow point which can be understood as the ``ray" cast from the vehicle to its shadow. We assume that initially $\| \vec{r} \|  = \sqrt{\left<\vec{r}, \vec{r}\right>}>0$, where $\left<\, , \, \right>$ stands for the inner product of vectors. The time derivative of $\| \vec{r}\|$ is
\begin{equation}
\begin{split}
    \frac{d}{dt} \| \vec{r} \| &= \frac{\left<\vec{r}, \dot{\vec{r}}\right>}{\|\vec{r}\|} = \left<\frac{\vec{r}}{\| \vec{r}\|}, (v_s \vec{x}_{s} - v \vec{x}_{v}) \right> \\
    &=v_{s} \left< \frac{\vec{r}}{\|\vec{r}\|}, \vec{x}_{s}\right> - v \left< \frac{\vec{r}}{\| \vec{r}\|}, \vec{x}_v\right> .
\end{split}
\end{equation}
As the shadow point is the closest point on the reference line to the vehicle, $\left<\vec{r}, \vec{x}_{s} \right> = 0$ and
if the reference line is on the right side of the vehicle:
\begin{equation}
    \vec{y}_{s} = -\frac{\vec{r}}{\| \vec{r}\|};
\end{equation}
and if the reference line is on the left side of the vehicle
\begin{equation}
    \vec{y}_{s} = \frac{\vec{r}}{\| \vec{r}\|}.
\end{equation}

As we will have an independent planner for the longitudinal speed $v_{s}$ along the reference line, we can decide the associated speed $v$ of the vehicle. Given $\forall t$ 
\begin{equation}
\begin{split}
    \frac{d}{dt} \left<\vec{r}, \vec{x}_{s} \right> &= 0 \\ 
                                        &= \left<\dot{\vec{r}}, \vec{x}_{s} \right> + \left<\vec{r}, \dot{\vec{x}}_{s}\right> \\ 
                                        &= v_{s}\left( 1 + \left<\vec{r}, \vec{y}_{s} \kappa\right>\right) - v\left<\vec{x}_{s}, \vec{x}_{v}\right>,
\end{split}
\end{equation}
we have
\begin{equation} \label{speed}
    v = \frac{v_{s}\left( 1 + \left<\vec{r}, \vec{y}_{s} \kappa \right>\right)}{\left<\vec{x}_{s}, \vec{x}_{v}\right>} .
\end{equation}
If a more complex vehicle model is applied to this methodology, for example, dynamic bicycle model, the velocity orientation can still be shown to be a locally controllable state variable in the orthnormal frame. The actuator which can vary the velocity orientation will not be the angular velocity of the front wheel $u$ in \eqref{eq:kbm}. Instead, it will be the steering angle of the front wheel.

\section{Optimality and Convergence}
\subsection{Optimality and convergence}
Using the idea of \textit{sliding mode control} \cite{Khalil}, with a constant $k>0$, we define a new state $e$ of the vehicle referring to the reference line and a new control $\tilde{u}$ associated to it where
\begin{equation}\label{eq:error_def}
    e \equiv \psi + \beta - \theta -  k \left<\vec{y}_{s}, \vec{r} \right>
\end{equation}
and 
\begin{equation}
\begin{split}
    \dot{e} &= \dot{\psi} + \dot{\beta} - \dot{\theta} - k \left<\dot{\vec{y}}_{s}, \vec{r}  \right> -  k \left<\vec{y}_{s}, \dot{\vec{r}} \right> \\
    &= \frac{v}{l_r} \sin(\beta) + g(\delta) u - \dot{\theta} + k v \sin(\psi + \beta - \theta) \\ 
    & = \tilde{u}.
\end{split}
\end{equation}
For steering of the front wheel $u$, $e$ is also locally controllable. To show the solution from following optimal control problem provides the comfortness of the passengers by minimizing the steering effort while the lateral deviation of the vehicle from the reference line is converging to zero, we decompose $u$ into $u = u_{s} + u_{c}$, where $u_{s}$ is the control to stabilize the orientation difference $\phi+ \beta - \theta$ at zero which is a complimentary steering effort and $u_{c}$ is the control to have $e$ to converge to zero.
$u_{s}$ is defined as
\begin{equation}
   u_{s} \equiv g(\delta)^{-1}\left(-\frac{v}{l_{r}} \sin(\beta) + \dot{\theta} - k v \sin(\psi+ \beta - \theta) \right)
\end{equation}
and 
\begin{equation}
    \dot{e} = \tilde{u} = g(\delta) u_c.
\end{equation}
We have a linear system with state $e$ and a control $\tilde{u}$, where $e$ describes the deviation of the vehicle away from the manifold $\psi + \beta - \theta - k \left<\vec{y}_{s}, \vec{r} \right> = 0$. To minimize this deviation, we have a simple LQR (Linear Quadratic Regulator) problem:
\begin{equation}\label{eq:LQR}
    \textrm{min}_{\tilde{u}} J = \int^{\infty}_{0} \left( \frac{e^{2}}{2} + \frac{\lambda}{2} \tilde{u}^{2}\right) dt.
\end{equation}
where $\lambda$ is a balance coefficient between the error and the steering effort to reduce this error about lateral deviation. The minimization of this steering effort $u^{2}_{c}$ improves the comfortness for the passengers.
With the solution of the ARE (Algebraic Riccati Equation), we have the optimal control
\begin{equation}
    \tilde{u}^{\star} = -\frac{1}{\sqrt{\lambda}} e.
\end{equation}
Using the bicycle model \eqref{eq:kbm}, the corresponding optimal control will be 
\begin{equation}
    \dot{\delta}^{*} = g(\delta)^{-1}\left(-\frac{1}{\sqrt{\lambda}}e + \dot{\theta} -\frac{v \sin(\beta)}{l_{r}} - kv \sin(\psi+ \beta - \theta)\right),
\end{equation}
if it is a feasible value of the front wheel angular velocity.  
\begin{theorem}
With the optimal control $\tilde{u}^{\star} = -\frac{1}{\sqrt{\lambda}} e$, both the lateral deviation of the vehicle from the target reference line (i.e. $\left<\vec{y}_{s}, \vec{r} \right>$) and the difference between the orientation of the vehicle velocity and the orientation of the shadow point (i.e. $\psi+ \beta - \theta$) will converge to zeros.
\end{theorem}
\begin{proof}
With the optimal control $\tilde{u}^{\star} = -\frac{1}{\sqrt{\lambda}} e$, we have $\dot{e} = -\frac{1}{\sqrt{\lambda}} e$, which indicates that $e$ converges to zero exponentially with the rate $\frac{1}{\sqrt{\lambda}}$ and $e(t) = e_{0} \exp\left(-\frac{t}{\sqrt{\lambda}}\right)$, where $e_0$ is the initial value of the error. With $e$ converging to zero, we have the state of the vehicle converge to the manifold:
\begin{equation}\label{eq:manifold}
    \psi + \beta - \theta - k \left<\vec{y}_{s}, \vec{r}\right> = 0.
\end{equation}
Meanwhile, with the feedback control, the dynamics of the orientation difference $ \Delta \theta \equiv \psi + \beta - \phi$ between the vehicle velocity and the tangent vector of the shadow point is:
\begin{equation}\label{eq:Delta_theta_dynamics}
    \dot{\Delta \theta} = -\frac{1}{\sqrt{\lambda}} \left( \Delta \theta - k\left<\vec{y}_{s}, \vec{r}_{2} \right> \right) - k v \sin(\Delta \theta)
\end{equation}
and take the time derivative on both sides of above equation again we have:
\begin{equation}\label{eq:second_order_system}
    \ddot{\Delta \theta} + \left(\frac{1}{\sqrt{\lambda}} + k v \cos(\Delta \theta) \right) \dot{\Delta \theta} + \frac{k v}{\sqrt{\lambda}} \sin(\Delta \theta) = 0 .
\end{equation}
Near the equilibrium point $(0, 0)$ of $(\Delta \theta, \dot{\Delta \theta})$, \eqref{eq:second_order_system} has two eigenvalues $(-\frac{1}{\sqrt{\lambda}})$ and $- k v$. As in the range $\Delta \theta \in (-\pi, \pi]$, $(0, 0)$ is the unique stable equilibrium point. The orientation of the vehicle will converge to the orientation of its shadow point. Going back to the manifold $\Delta \theta - k \left<\vec{y}_{s}, \vec{r}\right> = 0$, when $\Delta \theta = 0$, $\left<\vec{y}_{s}, \vec{r}\right> = \pm \| \vec{r}_{2} \|$ = 0. Therefore, with $\Delta \theta$ converging to zero, on the manifold \eqref{eq:manifold}, the deviation of the vehicle from the target lane will converge to zero.
\end{proof}

Apparently, in above theorem, the convergence rates at which lateral deviation and orientation difference both converge to zeros are $\frac{1}{\sqrt{\lambda}}$ and $k v$. Given the vehicle speed $v$, by varying the values of $\lambda$ and $k$, the convergence rates can be controlled and in a lane-change scenario, these convergence rates will describe how fast a lane-change maneuver will be completed by the vehicle. 

\subsection{Removing hazardous oscillation}
There are two independent coefficients $(k, \lambda)$ whose value we can choose. The dynamics of the lateral deviation is 
\begin{equation}\label{eq: lateral_dynamics}
    \frac{d}{dt} \left<\vec{y}_{s}, \vec{r} \right> = - v \sin(\Delta \theta) .
\end{equation}
Based on \eqref{eq:Delta_theta_dynamics},
\begin{equation}\label{eq:second_order_system_II}
    \frac{d^{2}}{dt^{2}} \left<\vec{y}_{s}, \vec{r} \right> +  k v \frac{d}{dt}  \left<\vec{y}_{s}, \vec{r} \right> = \frac{k v}{\sqrt{\lambda}} \cos(\Delta \theta) e_{0} \exp(-\frac{t}{\sqrt{\lambda}}).
\end{equation}
If we choose $\frac{1}{\sqrt{\lambda}} > k v$, the term on the right side of \eqref{eq:second_order_system_II} will decay faster than the dynamics of $\left<\vec{y}_{s}, \vec{r} \right>$ (i.e., $e$ is the fast mode of the system). The slow dynamics will be
\begin{equation}
    \frac{d^{2}}{dt^{2}} \left<\vec{y}_{s}, \vec{r} \right> +  k v \frac{d}{dt}  \left<\vec{y}_{s}, \vec{r} \right> = 0.
\end{equation}
In this way, the dynamics of $\frac{d}{dt}\left<\vec{y}_{s}, \vec{r} \right>$ will not be affected by the ``external force" $\frac{k v}{\sqrt{\lambda}} \cos(\Delta \theta) e_{0} \exp(-\frac{t}{\sqrt{\lambda}})$ in a long time. With $\frac{1}{\sqrt{\lambda}} > k v$, $e$ will first converge to $0$ as the fast mode and the time-derivative of the lateral deviation will converge to zero as the slow mode. In this way, potential hazard oscillation (in the space of time-derivative of the lateral deviation) will be removed. Here, we provide one example for this hazardous oscillation behavior. Given that the vehicle speed is $v = 1$ meter/second and $\lambda = 1$, from time $0$ to time $10$, there are three trajectories of the vehicle's lateral position of different $k$ values: $k = 0.5$, $k = 1$ and $k = 1.5$ in Figure \ref{fig:lateral_deviation} where all these trajectories converge to the target lane ($y = 3.5 \textrm{meters}$) from original target lane ($y = 0 \textrm{meter}$). In the plot of the time-derivatives of the lateral deviation along these trajectories in Figure \ref{fig:dynamics_lateral_deviation}, the hazardous oscillation with high value of $k$ ($k = 1.5$ and $k v > \frac{1}{\sqrt{\lambda}}$) appeared. Thus, to avoid this oscillation in path planning, given $v$ and a value of $\lambda$, we will choose $k$ such that 
\begin{equation}\label{eq:condition_1}
    k v = \frac{\lambda_{0}}{\sqrt{\lambda}}
\end{equation}
where $\lambda_{0} \in (0, 1)$.

\begin{figure}[t]
  \centering
  \includegraphics[scale=0.15]{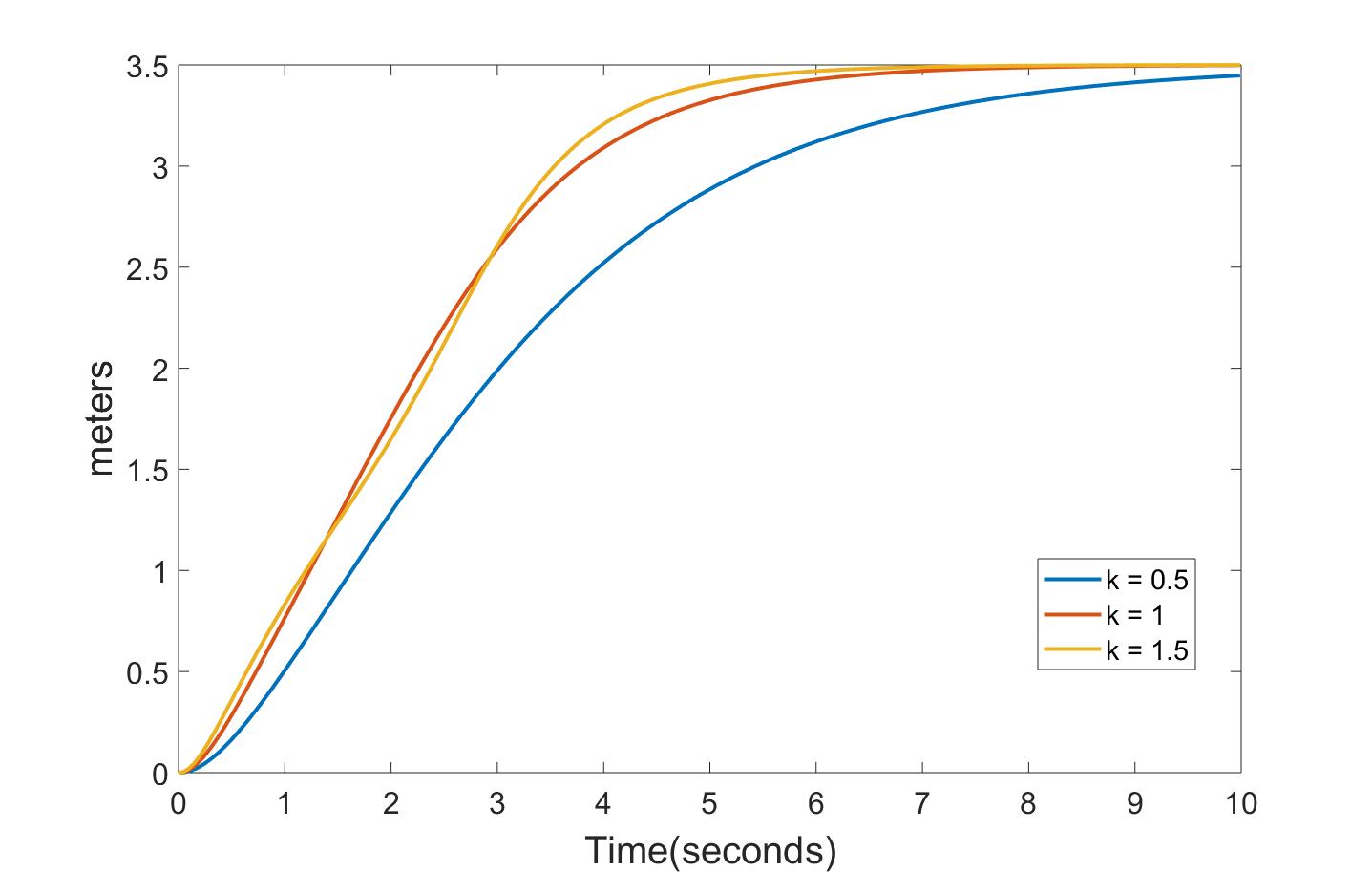}
  \caption{Lateral deviation of a lane-change maneuver}
  \label{fig:lateral_deviation}
\end{figure}
\begin{figure}
  \centering
  \includegraphics[scale=0.15]{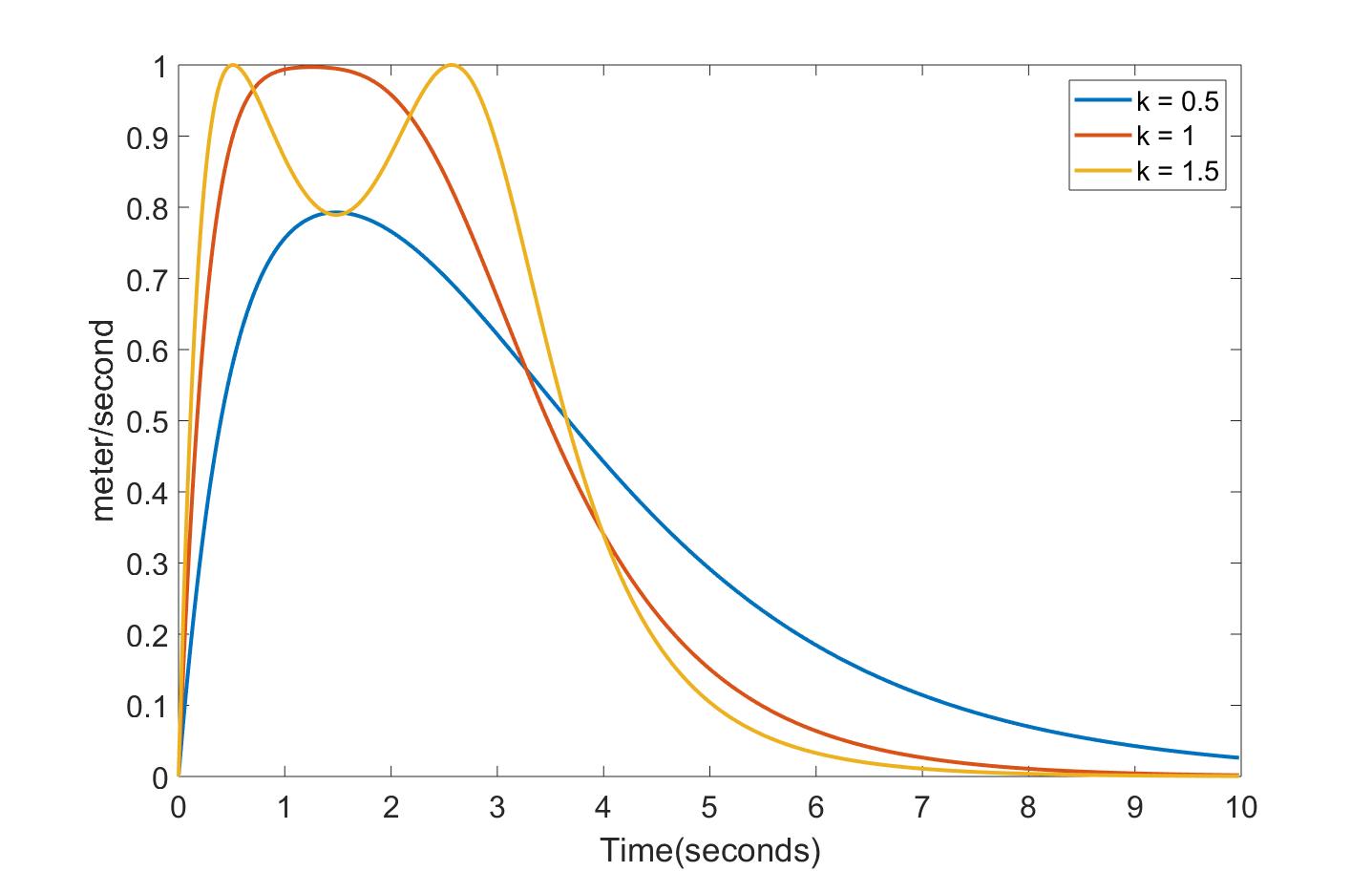}
  \caption{Time derivative of lateral deviation}
  \label{fig:dynamics_lateral_deviation}
\end{figure}

\section{Safety of Lane-change Maneuvers}

Lane-change maneuver is one task for this path planner. To drive the vehicle from one lane to another, we need to have a non-zero lateral velocity. This 
lateral velocity is generated due to the difference between the vehicle velocity orientation and the orientation of the corresponding shadow point as in equation \eqref{eq: lateral_dynamics}. When the surrounding traffic environment of the vehicle turns to be unsafe for our ego-vehicle to continue its current lane-change maneuver, it is necessary to have the vehicle to steer back to its original reference line quickly and safely. One necessary condition is that both $\|\Delta \theta\|$ and $\|\dot{\Delta \theta}\|$ can be bounded by setting up the values of $\lambda$ and $k$ in the process of a lane-change.  

Assuming that $\|\Delta \theta \| \ll 1$ and using \eqref{eq:condition_1}, the dynamics of $\Delta \theta$ and $e$ is a linear system:
\begin{align}
    \dot{e} &=-\frac{1}{\sqrt{\lambda}} e \nonumber \\
    \dot{\Delta \theta} &= -\frac{1}{\sqrt{\lambda}} e - \frac{\lambda_{0}}{\sqrt{\lambda}}\Delta \theta .
\end{align}
If the initial condition is $\Delta \theta (0) = 0$, 
\begin{align}
    \Delta \theta &= \left(\frac{e_{0}}{1-\lambda_{0}}\right) \left( \exp\left( -\frac{1}{\sqrt{\lambda}}t\right) - \exp\left(-\frac{\lambda_{0}}{\sqrt{\lambda}}t\right)\right) \nonumber \\
    \dot{\Delta \theta} &= -\frac{e_{0}}{\sqrt{\lambda}(1 - \lambda_{0})} \left(\exp \left(-\frac{t}{\sqrt{\lambda}}\right) -\lambda_{0} \exp\left(-\frac{\lambda_{0}}{\sqrt{\lambda}}t\right)\right)
\end{align}
where $e_{0}$ is the initial value of the error $e$, which only depends on the lateral deviation of the vehicle from the target lane (usually, its absolute value is proportional to the lane-width, based on the definition \eqref{eq:error_def}). As $\Delta \theta$ and $\dot{\Delta \theta}$ are linear combination of exponential functions, it indicates that each of them has a unique extremal along time $t$. To have the maximum values of $\|\Delta \theta \|$ and $\| \dot{\Delta \theta}\|$ are bounded by $C_{1}$ and $C_{2}$,
\begin{align}\label{eq:condition_2}
    \max \|\Delta \theta \| &= \frac{\|e_{0}\|}{\lambda_{0}} \exp \left( \frac{\ln{\lambda_{0}}}{1 - \lambda_{0}}\right) \le C_{1} \nonumber \\ 
    \max \|\dot{\Delta \theta} \| &= \frac{\|e_{0}\|}{\lambda_{0} \sqrt{\lambda}} \exp \left(\frac{2 \ln{\lambda_{0}}}{1 - \lambda_{0}}\right) \le C_{2}.
\end{align}
Given the vehicle speed $v$ is fixed, $k v = \frac{\lambda_{0}}{\sqrt{\lambda}}$ and $\|e_0\| = k W$, where $W$ is the lane-width, equations in \eqref{eq:condition_2} turn out that
\begin{align}\label{eq:condition_2_new}
    \frac{1}{\sqrt{\lambda}} \exp\left( \frac{\ln{\lambda_{0}}}{1 - \lambda_{0}} \right) &\le \frac{C_{1} v}{W} \\ \nonumber
    \frac{1}{\lambda} \exp\left( \frac{2\ln{\lambda_{0}}}{1 - \lambda_{0}} \right) &\le \frac{C_{2} v}{W}.
\end{align}
This safety necessary condition can be summarized as
\begin{equation}\label{eq:condtion_2_final}
    \frac{1}{\sqrt{\lambda}} \exp\left( \frac{\ln{\lambda_{0}}}{1 - \lambda_{0}} \right) \le \min \left\{\frac{C_{1} v}{W}, \sqrt{\frac{C_{2} v}{W}} \right\}.
\end{equation}
Based on constraints \eqref{eq:condition_1} and \eqref{eq:condtion_2_final} on parameters $\lambda_{0}$, $\lambda$ and $k$, given a vehicle speed $v$, we can have a path planner whose lane-change maneuver can be safely aborted and there will be no hazardous oscillation along the trajectories planned by this planner.

\section{Two-point Steering Path Planner}

The researchers who focus on the sensorimotor mechanism of human drivers noticed that people are using a two-point steering method to negotiate the curves along winding roads \cite{Land} \cite{Salvucci}. The general idea of \textit{two-point steering} is that the drivers utilize both the near and far regions of the road for guiding during steering.The information from near region is used to keep the vehicle in the center of the lane and the information from the far region facilitates the drivers to have a better performance while driving along a corner. One result of the two-point steering is ``corner-cutting" behavior: the drivers can utilize slight lateral deviation from the center of the lane to have the steering radius of the planned trajectory smaller than intrinsic radius of reference line of the lane. In some experiments and simulations (e.g., \cite{Sentouh} \cite{Okafuji}), researchers showed that using the idea of two-point steering, the planned trajectories possessed ``corner-cutting'' behaviors. Here, our work using geometric control, provides a rigorous theoretical foundation for ``corner cutting'' behaviors of two-point steering.  

In this section, we will implement the two-point steering based on the basic path planner designed in Section III and the behavior of ``corner-cutting" will be proved. The new state $e$ will be defined as:
\begin{equation}\label{eq:new_error}
     e \equiv \psi + \beta  - ((1-\alpha) \theta_{n} + \alpha \theta_{f}) - k \left<\vec{y}_{s}, \vec{r}\right>
\end{equation}
where $\theta_{n}$ is the orientation of the shadow point, $\theta_{f}$ is the orientation of the a way-point along the reference line along the direction of the vehicle and $\alpha \in [0, 1)$. 

With this novel error state $e$, we can still show that $e$ is locally controllable by the angular velocity of the front wheel and the solution of the optimal control problem as in \eqref{eq:LQR} still have the error state converge to zero in an exponentially decaying manner 
\begin{equation}
    \dot{e} = -\frac{1}{\sqrt{\lambda}} e.
\end{equation}
Using the theorem in Section III, if the ego-vehicle is tracking a straight lane (where $\theta_{n} = \theta_{f} = 0$), the vehicle will converge to the center of the lane. But, when the ego-vehicle negotiate a corner of constant curvature, it will exhibit a ``corner-cutting" behavior.
\begin{theorem}
    Using the two-point steering designed with the error state defined in \eqref{eq:new_error}, when the ego-vehicle is tracking a corner of constant curvature $\kappa_{0}$, with a appropriate set of parameters, the absolute value of the curvature of the vehicle trajectory $\|\kappa_{e}\|$ is smaller than $\|\kappa_{0}\|$. 
\end{theorem}
\begin{proof}
    The proof of the theorem will be divided into two parts. First we will prove that with right choice of parameters, along a lane of constant curvature, the deviation of the vehicle away from the lane center will be bounded. In consequence, it will be proven that the two-point steering path planner will result with a trajectory whose steering radius is smaller than that of the reference line of the target lane.  
    \begin{itemize}
        \item Step 1: As $e$ is stabilized near zero, $\theta_{v} \equiv \psi + \beta$ is the velocity orientation of the vehicle, $e = \theta_{v} - ((1 - \alpha)\theta_{n} +\alpha \theta_{f}) -  k \left<\vec{y}_{s}, \vec{r}\right>$,      and $\theta_{v} - \theta_{n} = \kappa_{0} \Delta d_{0}$, where $\Delta d_{0}$ is the \textit{look-ahead} distance from the shadow point to the way-point in far region,
            \begin{equation}\label{error_approx_zero}
                \theta_{v} - \theta_{n} - \alpha \Delta d_{0} \kappa_{0} =  k \left<\vec{y}_{s}, \vec{r}\right>.
            \end{equation}
            As $\|\theta_{v} - \theta_{n}\| \ll 1$, 
            \begin{equation}
                \frac{d}{dt} \left<\vec{y}_{s}, \vec{r} \right> = - v \sin(\theta_{v} - \theta_{n}) \approx  -v (\theta_{v} - \theta_{n}).
            \end{equation}
            Equation \eqref{error_approx_zero} can be re-written as
            \begin{equation}
                k \left<\vec{y}_{s}, \vec{r}\right>+ \frac{1}{v} \frac{d}{dt} \left<\vec{y}_{s}, \vec{r} \right> = -\alpha \Delta d_{0} \kappa_{0}.
            \end{equation}
            The Laplace transform of $\left<\vec{y}_{s}, \vec{r}\right>$ is 
            \begin{equation}
                \left<\vec{y}_{s}, \vec{r}\right>(s) = -\frac{\alpha \Delta d_{0} \kappa_{0}}{(\frac{s}{v} + k)s}.
            \end{equation}
            The steady-state error of the lateral deviation,
            \begin{equation}
            \begin{split}
                \lim_{t\rightarrow \infty}\left<\vec{y}_{s}, \vec{r}\right>(t) &= -\lim_{s\rightarrow 0} s\frac{\alpha \Delta d_{0} \kappa_{0}}{(\frac{s}{v} + k)s} \\
                &= -\frac{\alpha \Delta d_{0} \kappa_{0}}{k}.
            \end{split}
            \end{equation}
            Thus, to have the vehicle steered using two-point steering and staying in the lane,  it is necessary to have the absolute value of lateral deviation $\left \|\frac{\alpha \Delta d_{0} \kappa_{0}}{k}\right\|$ to be upper bounded. If we assuming the upper bound of the lateral deviation is a positive constant $C_3$, we have the first constraint for system parameters:
            \begin{equation}\label{eq:new_cond_1}
                \left \|\frac{\alpha \Delta d_{0} \kappa_{0}}{k}\right\| < C_3.
            \end{equation}
        \item Step 2: 
            As $\dot{e} = -\frac{1}{\sqrt{\lambda}} e$, we have
            \begin{equation}
                \ddot{e} = -\frac{1}{\sqrt{\lambda}} \dot{e}.
            \end{equation}
            Bringing the definition of $e$ into above dynamics and assuming that $\|\theta - \theta_{n}\| \ll 1$,
            \begin{equation}
            \begin{split}
            \ddot{\theta}_{v} + \left(\frac{1}{\sqrt{\lambda}} + kv \right) \dot{\theta}_{v} + \frac{kv}{\sqrt{\lambda}} \theta_{v} = \\
                (1-\alpha) \ddot{\theta}_{n} + \left( kv + \frac{(1 - \alpha)}{\sqrt{\lambda}}\right)\dot{\theta}_{n} + \frac{kv}{\sqrt{\lambda}} \theta_{n} \\
                     +\alpha \ddot{\theta}_{f} + \frac{\alpha}{\sqrt{\lambda}}\dot{\theta}_{f} .
            \end{split}
            \end{equation}
            Taking Laplace transform on both sides of above equation,
            \begin{equation}\label{eq:laplace}
            \begin{split}
                \left(s + \frac{1}{\lambda} \right) \left( s + kv \right) \theta_{v}(s) = \\
                (1-\alpha)\left(s + \frac{1}{\lambda} \right) \left( s + kv \right) \theta_{n}(s) \\
                + \alpha kv \left(s + \frac{1}{\sqrt{\lambda}} \right)\theta_{n}(s) + \alpha s (s + \frac{1}{\sqrt{\lambda}}) \theta_{f}(s).
            \end{split}
            \end{equation}
            As $-kv$ and $-\frac{1}{\sqrt{\lambda}}$ are two stable poles for $\theta_{v}$, we can implement zero-pole cancellation for \eqref{eq:laplace},
            \begin{equation}
                \theta_{v}(s) = (1-\alpha)\theta_{n}(s) + \alpha \theta_{f}(s) - \frac{\alpha kv }{s + kv} \left(\theta_{f} - \theta_{n}\right)(s).
            \end{equation}
            Taking inverse Laplace transform of above equation, in time-domain, we have
            \begin{equation}
            \begin{split}
                \theta_{v}(t) = (1 - \alpha) \theta_{n} + \alpha \theta_{f} \\
                -\alpha k v \int^{t}_{-\infty} \exp(-kv(t - \tau))\left(\theta_{f} - \theta_{n}\right)(\tau) d \tau
            \end{split}
            \end{equation}
            and the time-derivative of $\theta_{v}$ is
            \begin{equation}\label{eq:dot_angle}
                \dot{\theta}_{v} = (1 - \alpha) \dot{\theta}_{n} + \alpha \theta_{f} - \alpha k v \left(\theta_{f} - \theta_{n} \right).
            \end{equation}
            Using that fact: $\dot{\theta}_{v} = v \kappa_{e}$ and $\dot{\theta}_{n} = \dot{\theta}_{f} = v_{s} \kappa_{0}$,
            \begin{equation}\label{eq:curvature_relation}
                \kappa_{e} = \frac{v_{s}}{v} \kappa_{0} \left( 1- \alpha \frac{\lambda_{0}}{\sqrt{\lambda}} \frac{\Delta d_{0}}{v_{s}}\right).
            \end{equation}
            by following \eqref{eq:condition_1}: $k v = \frac{\lambda_{0}}{\sqrt{\lambda}}$, above equation about $\kappa_{e}$ is re-written as
            \begin{equation}
                \kappa_{e} = \frac{v_{s}}{v} \kappa_{0} \left( 1- \alpha k \Delta d_{0}\frac{v }{v_{s}}\right).
            \end{equation}
            Thereafter, to have $\|\kappa_{e}\| < \|\kappa_{0}\|$, we need
            \begin{equation}\label{eq:new_cond_2}
                 \alpha k \Delta d_{0} < \frac{v_{s}}{v}  < 1  + \alpha k \Delta d_{0}.
            \end{equation}
            Given the relation between the speed along reference line $v_{s}$ and the vehicle speed $v$ as \eqref{speed} and assuming $\cos(\theta_{v} - \theta_{n}) \approx 1$, bringing the steady-state lateral deviation into the relation, we have
            \begin{equation}
                \frac{v_{s}}{v} = \frac{1}{1 - \frac{\alpha \Delta d_{0} \kappa^{2}_{0} }{k}}.
            \end{equation}
            Bringing this speed ratio approximation into \eqref{eq:new_cond_2},
            \begin{equation}
                \alpha k \Delta d_{0} <  \frac{1}{1 - \frac{\alpha \Delta d_{0} \kappa^{2}_{0} }{k}}  < 1  + \alpha k \Delta d_{0}.
            \end{equation}
            Thus, this necessary condition to have $\|\kappa_{e}\| < \|\kappa_{0} \|$ is 
            \begin{equation}
                \frac{1}{1 + \alpha k \Delta d_{0}} < 1 - \frac{\alpha \Delta d_0 \kappa^{2}_{0}} {k} < \alpha k \Delta d_{0} < 1.
            \end{equation}
            Assuming $\gamma \in (0, 1)$ is a constant, such that $k \equiv \frac{\gamma}{\alpha \Delta d_{0}}$, the necessary condition above turned to be 
            \begin{equation}
                \|\kappa_{0}\| \sqrt{1 + \gamma} < k < \frac{\|\kappa_{0}\|}{\sqrt{\frac{1}{\gamma} - 1}}.
            \end{equation}
            To have $ \|\kappa_{0}\| \sqrt{1 + \gamma} < \frac{\|\kappa_{0}\|}{\sqrt{\frac{1}{\gamma} - 1}}$, the range of $\gamma$ is
            $\left( \frac{\sqrt{5} - 1}{2}, 1\right)$.
            Combining this result with \eqref{eq:new_cond_1} from Step 1, additional to constraints \eqref{eq:condition_1} and \eqref{eq:condtion_2_final}, the necessary condition for the ``corner-cutting behavior'' is 
            \begin{align}
                \max\left\{ \|\kappa_{0}\| \sqrt{1 + \gamma}, \sqrt{\frac{\gamma \|\kappa_{0}\|}{C_{3}}}\right\} < k < \frac{\|\kappa_{0}\|}{\sqrt{\frac{1}{\gamma} - 1}}
            \end{align}
            where $\gamma \in \left( \frac{\sqrt{5} - 1}{2}, 1\right)$ and $\alpha \Delta d_{0} k = \gamma$.
        \end{itemize}
\end{proof}
In summary, assuming that the lane-change maneuvers can only happen on the lanes of small steering radius, the safety constraints and removing hazardous oscillation constraints are still valid for this two-point steering path planner. To design the two-point steering path planner for lane-keeping and lane-change, there are following constraints for the parameters $(\alpha, \gamma, \lambda_{0}, \lambda, k, \Delta d_{0})$:
\begin{align}
    \alpha \Delta d_{0} k &= \gamma \in \left( \frac{\sqrt{5} - 1}{2}, 1\right), \nonumber \\
    \sqrt{\lambda} k v &= \lambda_{0} \in (0, 1) \nonumber \\
    \frac{1}{\sqrt{\lambda}} \exp\left( \frac{\ln{\lambda_{0}}}{1 - \lambda_{0}} \right) &\le \min \left\{\frac{C_{1} v}{W}, \sqrt{\frac{C_{2} v}{W}} \right\}, \nonumber \\
    \max\left\{ \|\kappa_{0}\| \sqrt{1 + \gamma}, \sqrt{\frac{\gamma \|\kappa_{0}\|}{C_{3}}}\right\} < k &< \frac{\|\kappa_{0}\|}{\sqrt{\frac{1}{\gamma} - 1}}.
\end{align}
where $C_1$, $C_2$, $W$ and $v$ are constants.

\section{Conclusion}
 In this paper, using geometric control, we designed a two-point steering path planner. One important assumption is that the vehicle velocity is a constant for a time duration. This assumption can be slacked to that the vehicle speed is smoothly varied with a small acceleration (deceleration). To clarify this slacked assumption will be a step-stone for combining this planner with a longitudinal speed planner. Moreover, this planner can be extended with obstacle-avoidance capability by conducting ``nudges'' on the target lane. Also, in this paper, the designed planner relied on the HD map for information such as shadow point and orientation and curvature of the shadow points. To implement the vision-based two-point steering using the theoretical foundation developed in this paper will be in our future work.  

\addtolength{\textheight}{-12cm}   



\end{document}